\documentclass[11pt]{article}
\usepackage[english]{babel}
\usepackage[letterpaper, portrait, hmargin=1.4in,vmargin=1in]{geometry}
\usepackage{amssymb, amsfonts, bbm, empheq, graphicx, float, mathrsfs,amsthm}

\usepackage{enumitem}
\usepackage{microtype}

\usepackage[dvipsnames]{xcolor}

\usepackage[hyphens]{url}
\usepackage[normalem]{ulem}

\usepackage{accents}
\usepackage{tabls}
\usepackage[linktocpage=true, colorlinks=true, allcolors=PineGreen,pdfusetitle]{hyperref}
\numberwithin{equation}{section} 

\usepackage[capitalise]{cleveref} 
\usepackage{autonum} 

\newtheorem{thm}{Theorem}[section]
\newtheorem{theorem}[thm]{Theorem} 

\newtheorem{lemma}[thm]{Lemma}

\usepackage{algorithm}
\usepackage{algpseudocode}
\algrenewcommand\algorithmicrequire{\textbf{Input:}}
\algrenewcommand\algorithmicensure{\textbf{Output:}}

\providecommand\given{}
\newcommand\SetSymbol[1][]{%
\nonscript\:#1\vert
\allowbreak
\nonscript\:
\mathopen{}}
\DeclarePairedDelimiterX\set[1]\{\}{%
\renewcommand\given{\SetSymbol[\delimsize]}#1}

\newcommand\E{\mathbb{E}} 
\newcommand\R{\mathbb{R}}

\newcommand\N{\mathbbm{N}}

\DeclareMathOperator*{\amin}{arg\,min}
\DeclareMathOperator*{\argmax}{arg\,max}
\DeclareMathOperator*{\argmin}{arg\,min}
\DeclarePairedDelimiter\norm{\lVert}{\rVert}
\DeclarePairedDelimiter\abs{\lvert}{\rvert}

\DeclarePairedDelimiter\lrp{(}{)}

\DeclarePairedDelimiter\brk{\langle}{\rangle}
\DeclarePairedDelimiterX\ip[2]{\langle}{\rangle}{#1,#2}

\renewcommand*{\P}{\mathbb{P}}

\DeclareMathOperator{\rk}{rk}

\DeclareMathOperator{\Tr}{tr}

\DeclareMathOperator{\im}{im}

\DeclareMathOperator\St{St}

\def\ddefloop#1{\ifx\ddefloop#1\else\ddef{#1}\expandafter\ddefloop\fi}
\def\ddef#1{\expandafter\def\csname #1scr\endcsname{\ensuremath{\mathcal{#1}}}}
\ddefloop ABCDEFGHIJKLMNOPQRSTUVWXYZ\ddefloop
\def\ddef#1{\expandafter\def\csname #1cal\endcsname{\ensuremath{\mathscr{#1}}}}
\ddefloop ABCDEFGHIJKLMNOPQRSTUVWXYZ\ddefloop

\def\ddef#1{\expandafter\def\csname #1hat\endcsname{\ensuremath{\widehat{#1}}}}
\ddefloop ABCDEFGHIJKLMNOPQRSTUVWXYZ\ddefloop
\def\ddef#1{\expandafter\def\csname #1bar\endcsname{\ensuremath{\bar{#1}}}}
\ddefloop ABCDEFGHIJKLMNOPQRSTUVWXYZ\ddefloop
\def\ddef#1{\expandafter\def\csname #1bar\endcsname{\ensuremath{\bar{#1}}}}
\ddefloop abcdefghijklmnopqrstuvwxyz\ddefloop

\def\ddef#1{\expandafter\def\csname #1bb\endcsname{\ensuremath{\mathbb{#1}}}}
\ddefloop ABCDEFGHIJKLMNOPQRSTUVWXYZ\ddefloop

\def\ddef#1{\expandafter\def\csname #1hat\endcsname{\ensuremath{\widehat{\csname #1\endcsname}}}}
\ddefloop {alpha}{beta}{gamma}{delta}{epsilon}{varepsilon}{zeta}{eta}{theta}{vartheta}{iota}{kappa}{lambda}{mu}{nu}{xi}{pi}{varpi}{rho}{varrho}{sigma}{varsigma}{tau}{upsilon}{phi}{varphi}{chi}{psi}{omega}{Gamma}{Delta}{Theta}{Lambda}{Xi}{Pi}{Sigma}{varSigma}{Upsilon}{Phi}{Psi}{Omega}{ell}\ddefloop

\newcommand{\by}{\times} 
\DeclarePairedDelimiterXPP\Exp[1]{\E}{[}{]}{}{\renewcommand\given{\SetSymbol[\delimsize]}#1}

\let\Pr\relax
\DeclarePairedDelimiterXPP\Pr[1]{\P}{[}{]}{}{\renewcommand\given{\SetSymbol[\delimsize]}#1}

\DeclarePairedDelimiterXPP\ind[1]{\mathbf{1}}{\{}{\}}{}{#1}

\DeclarePairedDelimiterXPP{\nucnorm}[1]{}\lVert\rVert{_\mathrm{nuc}}{#1} 
\DeclarePairedDelimiterXPP{\opnorm}[1]{}\lVert\rVert{_\mathrm{2}}{#1} 
\DeclarePairedDelimiterXPP{\Fnorm}[1]{}\lVert\rVert{_\mathrm{F}}{#1} 
\DeclarePairedDelimiterXPP{\enorm}[1]{}\lVert\rVert{_2}{#1}

\DeclareMathOperator{\rowspan}{rowspan}

\DeclareMathOperator{\spn}{span} 
\newcommand{\xpln}[1]{, \ \text{#1}}

\newcommand{\pseudo}[1]{ {#1}^\dagger}

\DeclareMathOperator{\Gr}{Gr}

\newcommand{\equalsd}{\overset{d}{=}} 
\newcommand\pX{\pseudo{X}}
\usepackage{cancel}

\newcommand\gtruth{\beta^{(0)}} 
\newcommand\iest{\beta^{(1)}} 
\newcommand\test{\bar\beta^{(1)}} 
\newcommand\fest{\beta^{(2)}} 
\newcommand{\est}[1]{\beta^{(#1)}} 
\newcommand{\truncest}[1]{\bar{\beta}^{(#1)}}
\newcommand\gtruthm{\Theta^{(0)}}
\newcommand\iestm{\Theta^{(1)}}
\newcommand\T{\intercal} 
\newcommand{\eqdef}{\coloneqq}
\DeclareMathOperator{\sg}{\textsf{subG}}

\title{\LARGE \bf
Estimation of Models with Limited Data\\by Leveraging Shared Structure
}

\author{Maryann Rui \and Thibaut Horel \and  Munther Dahleh%
\thanks{Laboratory for Information \& Decision Systems, Massachusetts Institute of Technology.
{\tt\small \{mrui,thibauth,dahleh\}@mit.edu}}%
}

\begin{document}

\maketitle

\begin{abstract}
Modern data sets, such as those in healthcare and e-commerce, are often derived from many individuals or systems but have insufficient data from each source alone to separately estimate individual, often high-dimensional, model parameters. If there is shared structure among systems however, it may be possible to leverage data from other systems to help estimate individual parameters, which could otherwise be non-identifiable. 
In this paper, we assume systems share a latent low-dimensional parameter space and propose a method for recovering $d$-dimensional parameters for $N$ different linear systems, even when there are only $T<d$ observations per system. To do so, we develop a three-step algorithm which estimates the low-dimensional subspace spanned by the systems' parameters and produces refined parameter estimates within the subspace. 
We provide finite sample subspace estimation error guarantees for our proposed method. Finally, we experimentally validate our method on simulations with i.i.d.\ regression data and as well as correlated time series data.
\end{abstract}

\section{Introduction}
In a variety of fields such as healthcare and e-commerce, it is often desirable to estimate parameters or provide recommendations for individuals based on data. Consider the common situation where we have $N$ different individuals, each with $T$ observations collected in $(X_i, Y_i) \in \R^{T\by d}\by \R^T$. Assume the data is generated as 
\begin{align}\label{eq:initmodel}
Y_i = X_i \gtruth_i + w_i \xpln{$i \in [N]$}
\end{align}
where $w_i \in \R^T$ is some independent noise vector, and $\gtruth_i$ are parameters of interest.

Such a linear model is ubiquitous in statistics, and standard least squares
regression provides an estimate of $\gtruth_i$ based on $(X_i, Y_i)$ when $T\geq d$ and $X_i$ is well-conditioned. 

Realistically, however, while a data set may contain many individuals, the data available from each individual may be limited, especially compared to the dimension of the parameter space considered. For instance, in the healthcare setting, patient data may be fragmented and stored on different electronic health record systems, so that each record system may have many individuals but imcomplete data from each \cite{dewdney2017electronic}. This may lead to problems of non-identifiability for individual systems, as in the case where we only have $T<d$ observations of a $d$-dimensional linear model. 

If there is a shared structure among individuals, however, it may be possible to leverage information from other individuals who share similar characteristics to overcome the challenge of non-identifiability of individual parameters. 

In this paper, we examine this possibility and propose a method of estimating each system's $d$-dimensional parameter by exploiting data from other systems along with the assumption that the parameters lie in a common $r$-dimensional subspace, where $r<d$. The questions we wish to answer are: can a sufficiently large number $N$ of systems compensate for a small amount $T$ of data per system in the task of estimating all the parameters? If so, how does the sample complexity scale in the parameters $N, T, r$ and $d$ of the problem?
 
If we simply count the degrees of freedom of the model, we have $r(d-r) + Nr$ parameters to estimate (the common $r$-dimensional subspace of parameters plus individual factor loadings or coefficients). Intuitively, one may expect that $NT \geq r(d-r) + Nr$ parameters are needed to jointly identify all parameters of the system. It is not obvious how to rigorously justify this intuition, nor how to develop and implement an estimation algorithm for this setting. 

To begin to tackle this complex and broad-ranging question, we propose an estimation method based on three separate least squares optimizations. 
The method first computes initial estimates of each system's parameter vector, which may be significantly far from the true value, but which still contain information about the common underlying subspace spanned by the true system parameters. Next, an estimate of this low dimensional subspace is obtained by extracting the top $r$ principal subspace of the first step estimates. From this, we can obtain a refined estimate of each system's individual parameters by solving another least squares problem, this time constrained to be over the estimated subspace. This last step requires $T\geq r$ for parameter identifiability, which can be a considerably easier condition to satisfy than the naive requirement of $T\geq d$, as $d \gg r$ in many real world datasets.  

 We provide finite sample subspace estimation error guarantees for a variant of our proposed method that takes into account the possible ill-conditioning of the pseudo-inverse-based least squares solution which arises when $T \approx d$. The analysis relies on obtaining concentration bounds for the sample covariance of the first-step estimates, and then proving that subspace estimation on these first-step estimates will obtain the true underlying subspace in expectation. 
 
Finally, we demonstrate our method and variants on simulations with i.i.d.\ regression data. We also evaluate our method on time series data with correlated regressors, and find that the method is flexible enough to handle this scenario. These results suggest the applicability of the three-step estimation method for more general settings of estimation of related with a common low rank structure. 

\subsection{Related Work}

\paragraph{Mixtures of linear regressions} The problem of estimating parameters
from limited observations of different systems that share a common low
dimensional structure is related to the problem of mixtures of linear
regressions \cite{faria2010fitting, li2018learning} and multitask, or
meta-learning \cite{kong2020robust, kong2020meta}. The main difference to our
setting is the systems' parameters are assumed to be clustered, rather than
coming from a low dimensional subspace.

\paragraph{Low rank matrix regression} Furthermore, a large body of work studies the related problem of low rank matrix regression, which usually uses a least squares estimator with nuclear norm regularization to estimate a low rank matrix. 
We can re-express \eqref{eq:initmodel} to make the comparison with matrix regression explicit. Let $\gtruthm:= \big[\gtruth_1 \cdots \gtruth_N\big]$ be the $d \by N$ matrix whose columns are the system parameters. Then by assumption, $\rk\gtruthm = r$. However, the data generating process for observations $(X_i, y_i)$ is
\begin{align}
y_i = X_i \gtruthm e_i + w_i \xpln{$i \in [N]$}. 
\end{align}
where $e_i$ is the $i$th coordinate vector in $\R^N$. While we are still trying to estimate a low rank matrix $\gtruthm$, the dimension of this matrix grows with the number of observations—unlike in matrix regression where it is assumed constant—landing us in a different regime for analysis and optimization.

\paragraph{Dictionary learning.} Finally, the problem of dictionary learning, also known as sparse coding, and matrix factorization, shares similar structure to the problem considered in this paper \cite{gr1, gr2, mairal2009online}. However, we only observe system parameters through the lens of a design matrix $X$ whose rows do not fully span the parameter space. Even if our design matrix $X \in \R^{T\by d}$ were the identity (so $T=d$), though, we also do not impose a sparsity assumption on the dictionary coefficients, as is standard in the dictionary learning literature. The differences are further detailed in Section \ref{sec:method1}.

\paragraph{Meta-learning and transfer learning.}
After the initial submission of this paper, we became aware of recent related work on this problem. In \cite{tripu2021provable}, the authors present a method of moments (MoM) estimator, and a similar estimator is studied under more general assumptions in \cite{duchi2022subspace}. We discuss the relationship between these estimators and ours and provide an empirical comparison in \cref{sec:comparison}.

More generally our work connects with the broader literature in machine
learning that studies learning across related tasks or data sets, usually
referred to by the umbrella terms meta-learning and transfer learning
\cite{pontil2015benefit,du2020few}.

\section{Preliminaries}

\subsection{General notations}

We define for $N\in\mathbb{N}^*$ the set $[N]:=\set{1, 2,\dots, N}$. The
inequality $a\lesssim b$ means that there exists a universal constant $C$
such that $a\leq Cb$.

For vectors $x,y\in\R^d$, $\ip x y \eqdef x^\T y=\sum_{i=1}^d x_i y_i$ and
$\norm{x}_2\eqdef \sqrt{\ip x x}$ denote the Euclidean inner product and norm,
respectively.

For a matrix $A$, $\Tr A$, $A^\T$ and  $\pseudo{A}$ denote its trace,
transpose and Moore–Penrose pseudoinverse, respectively. The identity matrix in
$\R^{d\by d}$ is written $I_d$. For matrices $A, B \in \R^{d\by r}$, $\brk{A,
B} := \Tr(A^\T B)$ denotes the Frobenius or trace inner product, 
$\Fnorm{A} = \sqrt{\Tr(A^\T A)}$ is the Frobenius norm of $A$, and $\opnorm{A}$
is its spectral norm.
Let $\Oscr(d)\eqdef\set{Q \in \R^{d \by d} \given Q^\T Q = QQ^\T = I_d}$ denote
the orthogonal group on $\R^d$ and $\St(r, d)\eqdef \set{A \in \R^{d\by r} \given A^\T A = I_r}$ denote the Stiefel manifold of orthonormal $r$-frames in $\R^d$.

\subsection{Subspaces}

For $r\leq d$, $\Gr(r, d)$ denotes the Grassmanian manifold of $r$-dimensional subspaces of $\R^d$ and we write $P_\Ascr$ for the orthogonal projection onto a subspace $\Ascr\in\Gr(r,d)$. For $Q$ an orthogonal projection of rank $r$ and $\Ascr\in\Gr(r,d)$, the identities $Q=P_{\im Q}$ and $\Ascr = \im P_\Ascr$ show that the map $\Ascr\mapsto P_\Ascr$ is a bijection from $\Gr(r,d)$ to the set of orthogonal projections of rank $r$. This allows us to identify the two sets
\cite[Sec.~1.3.2]{Chikuse2003}:
\begin{align}
	\Gr(r,d) \cong \set{P\in\R^{d\by d}\given P=P^\T=P^2\land \Tr P = r}.
\end{align}
Note that the choice of an orthonormal basis of a subspace $\Ascr\in\Gr(r,d)$
gives a representation of $\Ascr$ by an element $A\in\St(r,d)$, although the
representation is non-unique. For such a matrix $A$ we have $P_\Ascr=AA^\T$.

For any two subspaces $\Ascr, \Bscr \in \Gr(r,d)$, we can find $r$ pairs of
\emph{principal vectors} $(a_i, b_i) \in \Ascr \by \Bscr$ for $i \in [r]$, and
\emph{principal  angles} $(\theta_1, \dots, \theta_r) \in [0, \pi/2]^r$ such
that $\ip{a_i}{b_i} = \cos(\theta_i)$, $\ip{a_i}{a_j} = 0$, and $\ip{b_i}{b_j}
= 0$, for $i, j \in [r], i \neq j$. $\theta_1$ is the smallest angle between
any vector in $\Ascr$ and any vector in $\Bscr$, which is achieved by $a_1$ and
$b_1$. The remaining principal vectors and angles are defined inductively, by
restricting at each step to the orthogonal complement of the span of the
previous vectors \cite{golub2013matrix, ye2016schubert}. We write
$\Theta(\Ascr,\Bscr)$ for the diagonal matrix whose
diagonal entries are the principal angles. It is possible to show that the
nonzero eigenvalues of $P_\Ascr-P_\Bscr$ are the sines of the nonzero principal
angles between $\Ascr$ and $\Bscr$, each counted twice \cite[Sec.~VII.1]{bhatia}. This implies that
$\opnorm{P_\Ascr-P_\Bscr}=\opnorm{\sin\Theta(\Ascr,\Bscr)}$ and 
$\Fnorm{P_\Ascr-P_\Bscr}=\sqrt 2\Fnorm{\sin\Theta(\Ascr,\Bscr)}$.

\subsection{Random variables}

Unless otherwise specified, all random variables are defined on the same probability space. We write $X \equalsd Y$ for identically distributed variables $X$ and $Y$. We say that a random vector $X\in\R^d$ is sub-Gaussian with variance proxy $\sigma^2$, and write $X\in\sg_d(\sigma^2)$, if for all $\alpha\in\R^d$
\begin{equation}
	\E[\exp\ip\alpha X]\leq \exp\left(\frac{\sigma^2\enorm\alpha^2}{2}\right).
\end{equation}

\section{Model}\label{sec:model}

We consider $N$ linear systems of dimension $d \in \N$, from each of which we have $T \in \N$
observations.
Specifically, each system $i\in[N]$ has observations $(X_i, Y_i) \in \R^{T\by
d}\by \R^T$ generated according to 
\begin{equation}\label{eq:model}
Y_i = X_i \gtruth_i + w_i
\end{equation}
where $\gtruth_i\in\R^d$ is the parameter of system $i\in[N]$.

The central assumption of our model is that the system parameters
$(\gtruth_i)_{i\in[N]}$ lie in an $r$-dimensional subspace $\Bscr_0$ of
$\R^d$, where $1 \leq r \leq d$. An orthonormal basis $B_0\in\St(r,d)$ of $\Bscr_0$ constitutes a common dictionary
of $r$ atoms shared by all systems and we can write for each $i\in[N]$
\begin{equation}\label{eq:model-2}
	\gtruth_i = B_0\phi_i
\end{equation}
for some $\phi_i\in\R^r$.
We are primarily interested in estimating the $r$-dimensional
subspace $\Bscr_0$, from which we can easily recover $\phi_i$,
and then $\gtruth_i$ for $i \in [N]$, from the data.

We now describe our distributional assumptions. All the random variables $(X_i)_{i\in[N]}$, $(\phi_i)_{i\in[N]}$ and $(w_i)_{i\in[N]}$ are mutually independent. The matrices $(X_i)_{i\in[N}$ are identically distributed, each having i.i.d.\ standard normal entries. Coefficients $\phi_i \in \R^r$ and noise $w_i \in \R^T$ are centered isotropic sub-Gaussian random vectors in $\sg(\sigma_\phi^2)$ and $\sg(\sigma_w^2)$, respectively, with covariance matrices $\sigma_\phi^2 I_r$ and $\sigma_w^2 I_T$, respectively.

\section{Three-Step Estimator}\label{sec:method}

We propose an estimation method that follows a general three step approach.
First, compute initial estimates of each system $i$'s parameter. Next, find the
$r$-dimensional subspace that best explains the initial estimates. In the last
step, the individual system estimates are refined by leveraging the subspace
learned in the second step.

Due to the linear structure of our observation model, instantiating the above
approach naturally results in formulating and solving a linear least squares
problem at each of the three steps. The following subsections describe these
least squares problems in more details and \cref{alg:norm} summarizes the
computation of their solution.

\begin{algorithm}[H]
\caption{Three-step parameter and subspace estimator}\label{alg:norm}
\begin{algorithmic}[1]
\Require Samples $(X_i, Y_i) \in \R^{T \by d} \by \R^T$, for $i \in [N]$; model rank $r$.
\Ensure Estimated subspace frame $\Bhat \in \St(r,d)$, coefficients $\set{\est2 \in \R^d: i \in [N]}$.

\For{$i \in [N]$}
	\State $\est1_i \gets \pX_i Y_i$ 		\Comment{First-step estimate}
	\State $\truncest1_i \gets \est1_i/\enorm{\est1_i}$
	\Comment{Normalization}\label{line:norm}
\EndFor

\State $USV^\T \gets $ SVD$([\truncest1_1 \cdots \truncest1_N])$
\State $\Bhat \gets U[:, 1:r]$ 
\Comment{Subspace estimation}
\For{$i \in [N]$}
	\State $\est2_i \gets \Bhat \pseudo{(X_i\Bhat)} Y_i$ \Comment{Coefficient estimation}
\EndFor
\end{algorithmic}
\end{algorithm}

\subsection{Initial individual estimates}\label{sec:method1}
We first obtain a least squares estimate $\iest_i$ of $\gtruth_i$ up to the
null space of $X_i$, which is nontrivial as we focus on the regime $T<d$. Let $\pX_i$ be the pseudoinverse of $X_i$. Then our initial estimates are
\begin{align}\label{eq:firststep}
\iest_i = \pX_i Y_i \xpln{$i \in [N]$}.
\end{align}
To gain a better understanding of our initial estimates $\iest_i$, under our
observation model \eqref{eq:model}, we can write
\begin{align}
	\iest_i = \pX_i X_i\gtruth_i + \pseudo{X_i} w_i
	= P_{X_i}\gtruth_i + \pseudo{X_i} w_i
\end{align}
where $P_{X_i}\eqdef \pseudo X_i X_i$ is the $d\by d$ projection matrix onto the
$T$-dimensional row space of $X_i$. Since the distribution of the rows of $X_i$
is rotationally invariant, one can check that $\E[P_{X_i}] = \frac T d I_d$
(see \cref{lemma:uniform-g,lem:1}).
Hence, up to a rescaling by $\frac d T$, $\iest_i$ is an unbiased estimate of
$\gtruth_i$. However, the noise of this estimate
\begin{equation}\label{eq:beta0}
	\eta_i\eqdef\iest_i-\gtruth_i = -(I_d-P_{X_i})\gtruth_i + \pseudo X_i w_i
\end{equation}
is not independent of the true parameter $\gtruth_i$ as it includes its
projection onto the null space of $X_i$ that was left unobserved.

\paragraph{Normalization.} The normalization step on line~\eqref{line:norm}
ensures that the first-step estimates are all weighted equally in the subspace
estimation step. As will become clear in the simulations
(\cref{sec:simulation}) this mitigates issues arising due to the pseudo-inverse
$\pseudo X_i$ being ill-conditioned when $T$ is close to $d$. For tractability
reasons, our theoretical analysis studies a variant of \cref{alg:norm} in which
the normalization step is replaced with a truncation $\truncest1_i \gets
\est1_i\ind{\opnorm{\pX_i} \leq s}$, for some predefined threshold $s$. We also
compare this variant to our main estimator in \cref{sec:simulation}.

\paragraph{Comparison with dictionary learning.} 
Using \eqref{eq:beta0} wen can write $\iest_i=\gtruth_i+\eta_i$. Hence, our setting is reminiscent of the dictionary learning, or sparse coding, problem, in which we have $N$ noisy observations where $\gtruth_i = B_0 \phi_i$, for some fixed $d \by r$ matrix $B_0$, and some vector $\phi_i \in \R^r$. Both $B_0$ and $\phi_i$ are unknown and the goal is to learn $B_0$, which in our setting then allows for a straightforward estimation of $\phi_i$ and $\gtruth_i$.

However, we cannot apply dictionary learning methods straight off the shelf. First, dictionary learning models assume sparsity of the unknown coefficients $\phi_i$, and for most sample complexity results, a degree of sparsity is necessary (i.e., the size of the support of $\phi_i$ is upper bounded) \cite{gr1, arora2014new, gribonval2010dictionary}. Meanwhile, we make no restrictive assumptions on the factor loadings $\phi_i$. 
Second, the additive noise $\eta_i$ in dictionary learning is assumed to be independent of other randomness in the problem, either with standard subgaussian or bounded distributional assumptions  \cite{gr1, arora2014new, gr2}. However, as already mentioned, $\eta_i$ in our method is visibly not independent of the parameters $\gtruth_i$ to be estimated, as it contains the component of $\gtruth_i$ in the null space of $X_i$. 

\subsection{Subspace recovery}

The goal of the second step is to compute an estimate of the $r$-dimensional subspace of $\R^d$ containing the ground truth parameters $(\gtruth_i)_{i\in[N]}$. We do so by finding the $r$-dimensional subspace $\hat{\Bscr}$ that best approximates the (normalized) first step estimates $(\test_i)_{i\in[N]}$, in the least squares sense. If $P$ denotes the orthogonal projection onto this optimal subspace, the residual error associated with $\iest_i$ is its distance to the subspace, that is $\norm[\big]{\test_i-P\test_i}_2$. Consequently, the least squares problem at this step is
\begin{align}\label{eq:2nd-step}
\min_{P\in\Gr(r, d)}\frac 1 N\sum_{i\in[N]}\norm[\big]{\test_i-P\test_i}_2^2.
\end{align}
Note that \eqref{eq:2nd-step} is exactly the problem of finding the space
spanned by the first $r$ \emph{principal components} of the first step
estimates $(\test_i)_{i\in[N]}$. Those are given by the top $r$ left-singular
vectors of the matrix $\iestm\in\R^{d\times N}$ whose columns are
$(\test_i)_{i\in[N]}$.

Another interpretation of this subspace estimate can be obtained by observing
that:
\begin{align}
\sum_{i\in[N]}\enorm[\big]{\test_i-P\test_i}^2 &=\Fnorm{\iestm-P\iestm}^2\\[-1em]
&=\ip[\big]{(I_d-P)\iestm}{(I_d-P)\iestm}\\
&=\ip*{I_d-P}{\iestm{\iestm}^\T},
\end{align}
where the last equality uses that $I_d-P$ is also an orthogonal projection.
This allows us to rewrite \eqref{eq:2nd-step}
\begin{align}
	\argmin_{P\in\Gr(r, d)}\frac 1
N\sum_{i\in[N]}\enorm[\big]{\test_i-P\test_i}^2
	&=\argmax_{P\in\Gr(r,d)} \ip*{P}{\frac{\iestm{\iestm}^\T}N}.
\end{align}
The matrix $\frac 1 N\iestm{\iestm}^\T=\frac 1 N\sum_{i\in[N]}\test_i{\test_i}^\T$ appearing on the last line is the sample covariance matrix of the first step estimates. Because this matrix is positive semi-definite, it admits a spectral decomposition with non-negative eigenvalues and orthogonal eigenspaces. The top $r$ left-singular vectors of $\iestm$ are equivalently given by an orthonormal collection of eigenvectors associated with the top $r$ eigenvalues of $\iestm{\iestm}^\T$ (counted with multiplicity).

\subsection{Parameter recovery}
In the third stage of our algorithm we obtain revised estimates of $\fest_i$ for each $i \in [N]$ given the estimate $\hat{\Bscr}$ with orthogonal frame matrix $\Bhat \in \St(r,d)$:
\begin{align}\label{eq:refinedbeta}
\fest_i &\in \amin_{\beta \in \hat{\Bscr}} \enorm{y_i - X_i \beta}^2 \\
&= \Bhat \pseudo{(X_i \Bhat)}y_i = P_{\Bhat} \pseudo{(X_i P_{\Bhat})} y_i.
\end{align}
One can obtain this result by solving for $\fest_i = \Bhat \phihat_i$ where 
\begin{align}
    \phihat_i \in \amin_{\phi \in \R^r} \enorm{y_i - X_i \Bhat \phi}^2.
\end{align}

\section{Results}\label{sec:results}

\subsection{Sample complexity}

Our main theoretical result is an upper-bound on the sample complexity of the estimator described in \cref{alg:norm}. As already mentioned, it is easier to analyze a variant in which line~\eqref{line:norm} performs a truncation instead of a normalization. Thus, for the remainder of this section $\test_i$ is defined as $\test_i=\iest_i\ind{\opnorm{\pX_i} \leq s}$, for some predefined threshold $s$. Equivalently, \cref{alg:norm} simply drops the first step estimates $\iest_i$ for which $\opnorm{\pX_i}> s$, and uses the remaining ones in the subspace estimation step.

\begin{theorem}\label{prop:sample-complexity}
	Let $\hat\Bscr$ be the subspace spanned by the columns of the output $\Bhat$ in \cref{alg:norm} with threshold level $s=\Omega\big(1/(\sqrt d -\sqrt{T-1})\big)$ if $T\leq d$, and $s =\Omega\big(1/(\sqrt T -\sqrt d)\big)$ if $T>d$. Then for each $0<\delta<1$, with probability at least $1-\delta$
	\begin{align}
		\opnorm{\sin \Theta(\hat\Bscr,\Bscr_0)}
		\lesssim \begin{cases}
\left(1+\frac{\sigma_w^2/\sigma_\phi^2}{(\sqrt d - \sqrt
		{T-1})^2}\right)
	\frac{d^2}{T^2}
		\sqrt{\frac d N\log\frac 2\delta}\ &\text{if $T\leq d$} \\
 \left(1+\frac{\sigma_w^2/\sigma_\phi^2}{(\sqrt T - \sqrt d)^2}\right)
		\sqrt{\frac d N\log\frac 2\delta}\ &\text{if $T > d$.}
		\end{cases}
	\end{align}
\end{theorem}

The proof of \cref{prop:sample-complexity} is provided in \cref{app:proof}. We first show that the $r$th principal subspace of the covariance matrix $\Sigma_\beta=\E[\test_i{\test_i}^\T]$ is the ground truth subspace $\Bscr_0$ and quantify its spectral gap. By the Davis–Kahan theorem, upper-bounding the principal subspace angle reduces to upper-bounding the spectral norm of $\frac 1 N \sum_{i=1}^N \test_i{\test_i}^\T - \Sigma_\beta$. This follows from a standard result on the concentration of covariance matrices of sub-Gaussian random variables.

The bound in \cref{prop:sample-complexity} scales as $1/\sqrt{N}$ as expected
when learning from independent observations. While the dimension $r$ of
$\Bscr_0$ does not explicitly appear in the bound, it would be conventional to
take $\sigma_\phi^2=1/r$, so that the norm of the ground truth parameters
$\gtruth_i$ concentrate around $1$, as is usually assumed in sample complexity
bounds. For this choice of $\sigma_\phi^2$, our error bound scales linearly in
$r$. 
The first term in the bound degrades as $T$ increases and becomes $\Omega(d)$
when $T=d$. This is due to the design matrix $X_i$ becoming
ill-conditioned for $T$ close to $d$, a phenomenon we inspect more closely
in \cref{sec:simulation} below.
When $T>d$, the error bound has a rather negligible dependence on $T$, while one would hope for a bound that decreases in $\sqrt{T}$. Obtaining tighter bounds is left for future work.

\subsection{Simulation}\label{sec:simulation}
We investigate by simulation the performance of the three-step estimator of \cref{alg:norm} that uses normalized first-step estimates, as well as the thresholding variant of \cref{alg:norm} analyzed in \cref{prop:sample-complexity}, and which uses a truncated first-step estimate for subspace estimation. 

For each trial, an $r$-dimensional subspace $\Bscr_0\subset \R^d$ is fixed and i.i.d.\ samples $(X_i, Y_i)$, $i \in [N]$
 are generated according to \eqref{eq:initmodel}, with $w_i\sim \Nscr(0, \sigma_w^2 I_T)$, $\est0_i$ drawn
 uniformly from the unit ball intersected with $\Bscr_0$ (and thus subgaussian with variance proxy $\sigma_\phi^2 = 1/r$), and $X_i$ with i.i.d.\ standard normal entries. Each plot shows the average of 30 trials and error bars indicate one standard deviation.

Figure \ref{fig:our_trendN} shows the subspace estimation error of the estimate $\Bhat$ from (1) \texttt{thresh}, \cref{alg:norm} with truncated first-step estimates, (2) \texttt{norm}, \cref{alg:norm} with normalized first step estimates. Results are shown for $d=50, r=5$, in regimes $d>T=10$, $T=d=50$, and $d<T=80$, as the number of systems $N$ is varied.

\begin{figure}[t]
\centering
\includegraphics[width=0.63\textwidth]{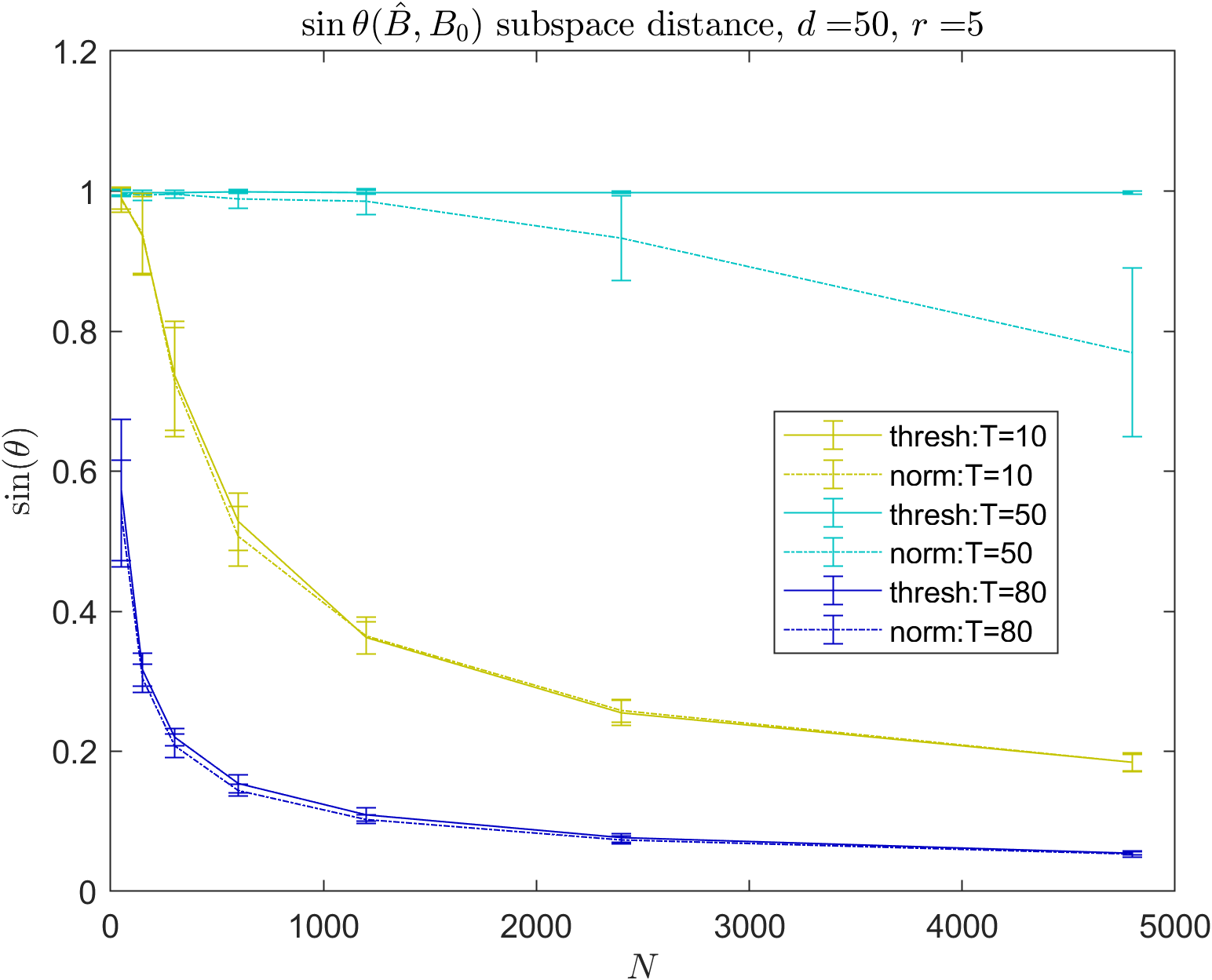}
\caption{Subspace estimation error vs.\ $N$ for estimates $\Bhat$ from (1) \texttt{thresh} and (2) \texttt{norm}. Results are shown for $d=50, r=5$, in regimes $d>T=10$, $T=d=50$, and $d<T=80$.} 
\label{fig:our_trendN}
\end{figure}

Both versions of the three-step estimator do well in both the $T>d$ and $T<d$ regime, though we note suboptimal performance when $T=d$. This arises from the fact that the pseudoinverse of $X_i$ can be ill-conditioned when $T=d$. 
As described before, we address this issue by normalizing our first step estimates to mitigate the effect of a single sample misdirecting the subspace estimator with an amplified noise term $\pX_i w_i$. For the truncating estimator, we chose an optimal threshold level for each $d$ and $T$ that trades off controlling the effect of possibly ill-conditioned pseudo-inverse-based least squares estimates with losses in effective sample size.

\begin{figure}[H]
\centering
\includegraphics[width=0.72\textwidth]{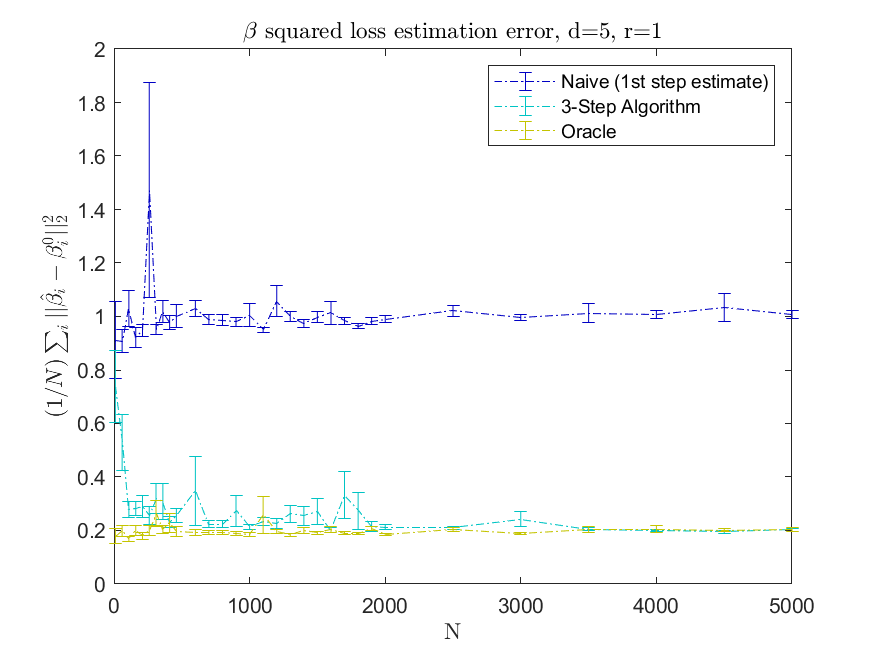}
\caption{Comparison of $\beta_i$ estimation error of \texttt{thresh} as compared to oracle and naive estimators. Here, $d = 5, r = 1, T=3$.}
\label{fig:beta_est}
\end{figure}

In Figure \ref{fig:beta_est}, we show the performance of the refined estimates $\fest_i$ from \texttt{thresh} in comparison with two benchmarks: (1) the ``oracle'' least squares estimate assuming $\Bscr_0$ is known, which is obtained as the minimizer of \eqref{eq:refinedbeta} with $\hat{B}$ replaced by the true subspace $\Bscr_0$, and
 (2) the naive least squares estimate run separately for each system $i \in [N]$, and which does not share information across systems, i.e., $\est1_i$ in \eqref{eq:firststep}. The \texttt{thresh} estimator is able to leverage information across systems to eventually match the performance of the oracle.

\subsection{Comparison with related work}\label{sec:comparison}

In \cite{tripu2021provable}, the authors present a method of moments estimator, which we refer to as \texttt{MoM}, for subspace estimation in the present setting. This estimator can be interpreted under our three-step method as first obtaining a first-step estimate $\est1_i = X_i^\T Y_i$ and then estimating the subspace shared by these first-step estimators. Specifically, the matrix $X_i^\T$ pre-multiplies $Y_i$ rather than the pseudoinverse $\pX_i$ pre-multiplying $Y_i$ as it does in the first-step least squares estimate of \cref{alg:norm}.

\begin{figure}[H]
\centering
\includegraphics[width=0.63\textwidth]{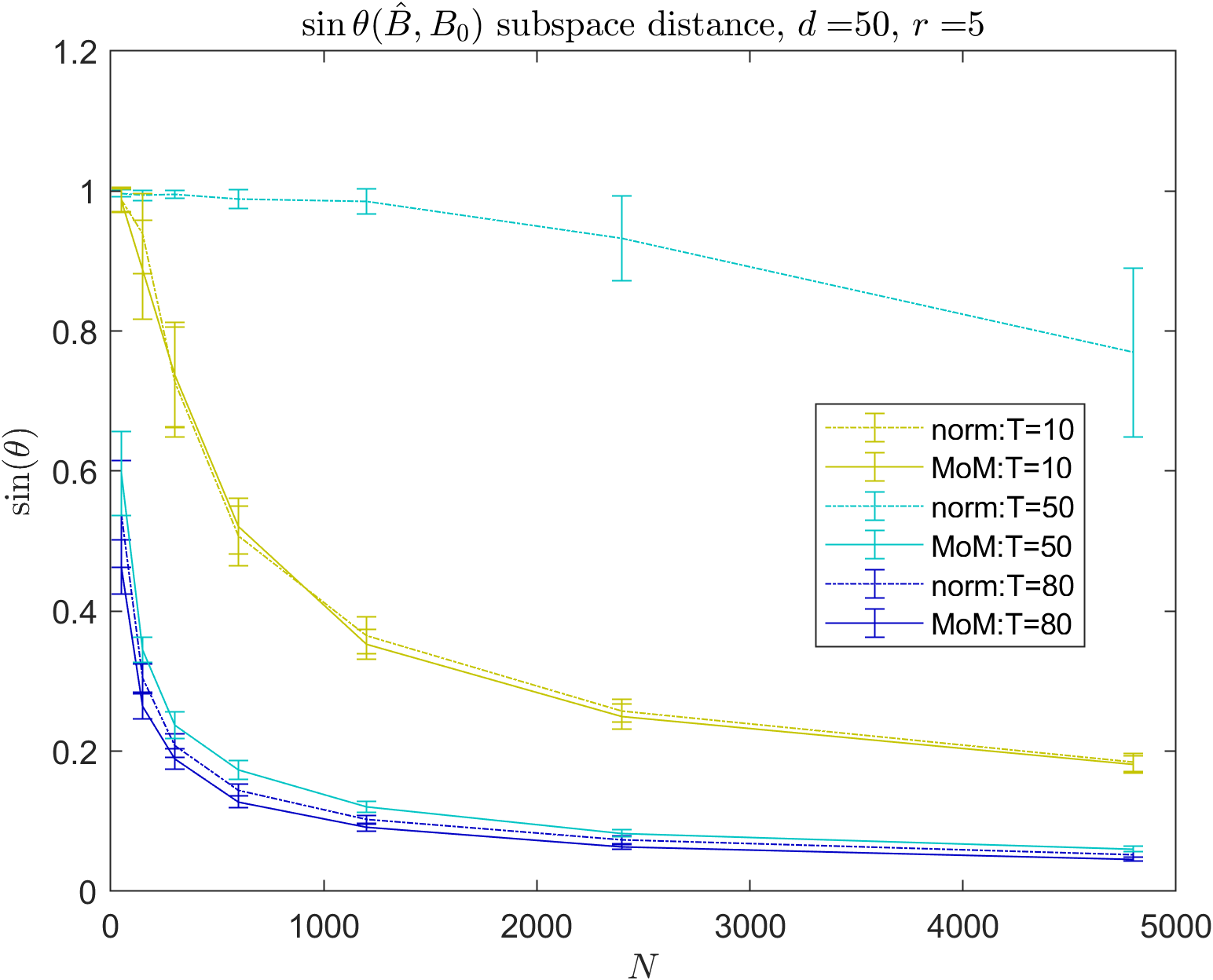}
\caption{Subspace estimation error vs.\ $N$ for estimates $\Bhat$ from (1)
\texttt{norm} and (2) \texttt{MoM}, for i.i.d.\ data. Results are shown for $d=50, r=5$, in regimes $d>T=10$, $T=d=50$, and $d<T=80$..}
\label{fig:compare_trendN}
\end{figure}

\begin{figure}[H]
\centering
\includegraphics[width=0.63\textwidth]{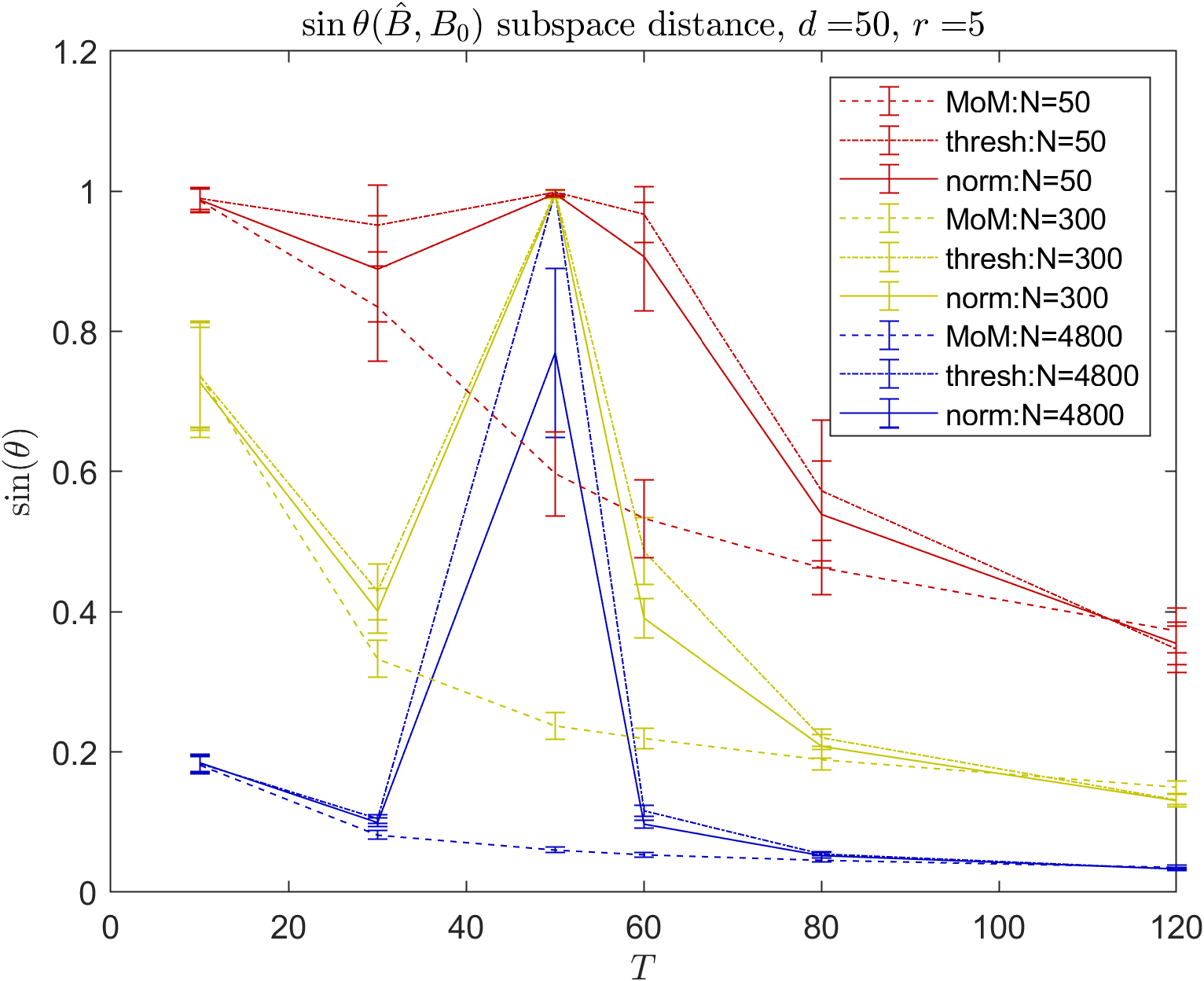}
\caption{Subspace estimation error vs.\ $T$ for estimates $\Bhat$ from (1)
\texttt{thresh}, (2) \texttt{norm},  and (3) \texttt{MoM}, for i.i.d.\ data. Results are shown for $d=50, r=5$, for various values of $N$.}
\label{fig:compare_trendT}
\end{figure}

Figure \ref{fig:compare_trendN} and \ref{fig:compare_trendT} compare the
performance of \texttt{norm} and of \texttt{MoM} in terms of subspace
estimation error. Results are shown for the regimes $T<d$, $T=d$, and $T>d$. We
see that both estimators perform comparably in the first and third regimes,
while \texttt{norm} suffers in the $T=d$ regime where $\pX$ may be
ill-conditioned. However, the next section on time series data suggests that
\texttt{norm} may generalize better to settings with dependent regressors such as time-series data.

\subsection{Time series estimation}
We next evaluate our algorithm \texttt{norm} and the method of moments estimator, \texttt{MoM} of \cite{tripu2021provable}, on time series data. Specifically, consider $T$ observations $(x_{it})_{t \in \set{0, 1, \dots, T}}$ generated as: 
\begin{align}
x_{i,t+1} = A_i x_{it} + w_{it},\ t \in [T-1],\ x_{i0} \sim \Nscr(0, \sigma_x^2) 
\end{align}
for each $i \in [N]$, with $x_{it} \in \R^d$, $A_i \in \R^{d \by d}$ and
$w_{it}\in \R^d$ a sub-Gaussian random vector in $\sg(\sigma_w^2)$. We assume each dynamics matrix $A_i$ is of rank $r\leq d$ and can be written in the form $A_i = F_i B^\T/ \opnorm{F_i B^\T}$,
where the rows of $F_i \in \R^{d\by r}$ are independently distributed and rotationally invariant, and $B \in \St(r, d)$ is an orthonormal $r$-frame for a subspace $\Bscr \in \Gr(r,d)$. Thus, the rows of $A_i$ lie in $\Bscr$, for $i \in [N]$. $F_i B^\T$ is normalized by its operator norm to ensure stability, since the operator norm dominates the magnitude of the eigenvalues. 
Unlike the setting of i.i.d.\ regression, the covariates $(x_{it})_{t \in [T]}$ are no longer independent of each other and of the collection of noise vectors $(w_{it})_{t \in [T]}$. 

As demonstrated in Figure \ref{fig:timeseries}, \texttt{norm} is able to generalize to this setting quite well when $T$ is not close to $d$, as opposed to the \texttt{MoM}, which fails to learn even as $N$ increases. Intuitively, \texttt{MoM} is not robust to the non-isotropy of the regressors $X$, while our least-squares-based first-step estimate is still able to extract useful information in this setting. 

\begin{figure}[H]
\centering
\includegraphics[width=0.63\textwidth]{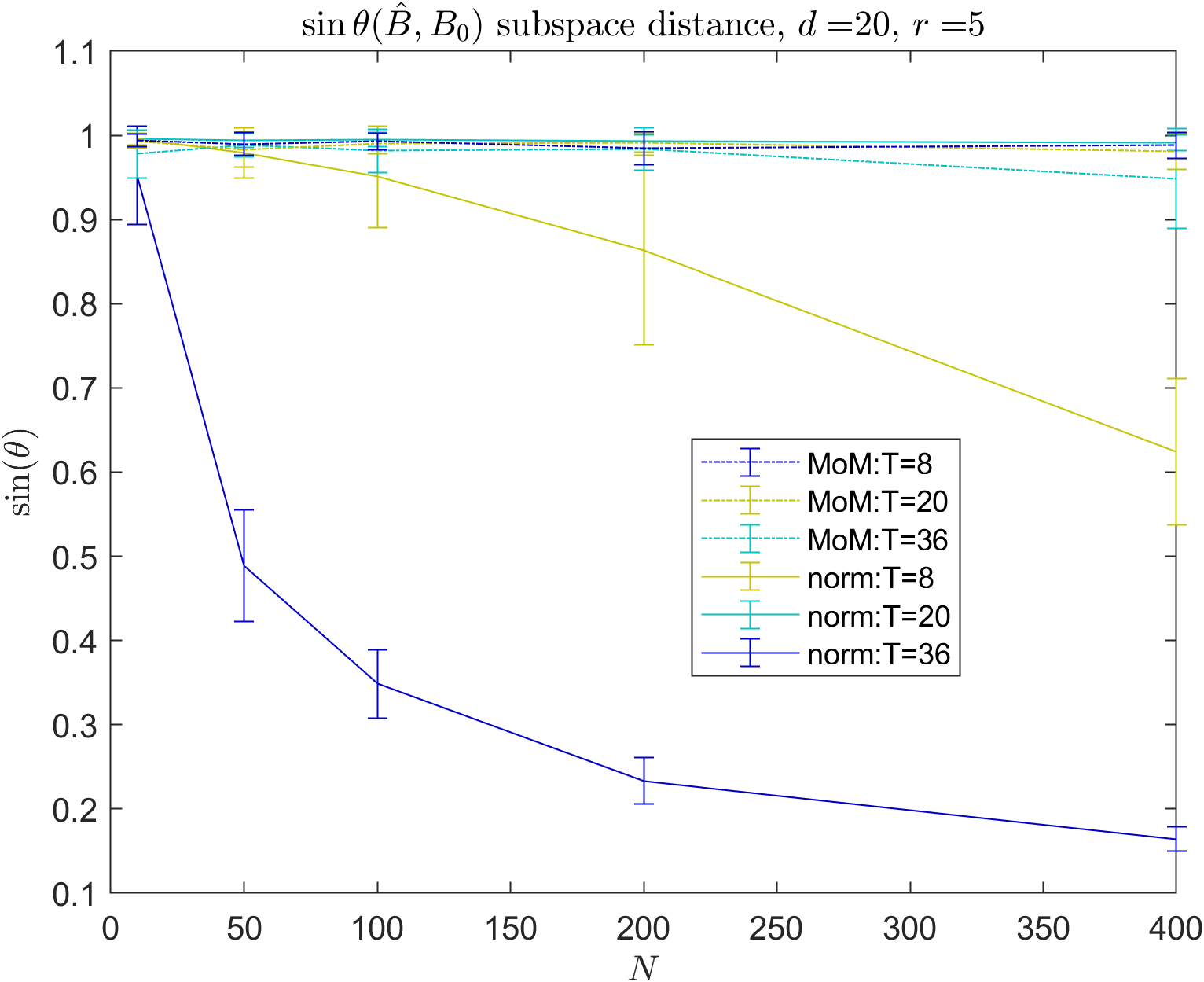}
\caption{Subspace estimation error vs.\ $N$ for estimates $\Bhat$ from (1) \texttt{norm} and (2) \texttt{MoM}, for time-series data. Results are shown for $d=20, r=5$, in regimes $d>T=8$, $T=d=20$, and $d<T=36$.}
\label{fig:timeseries}
\end{figure}

\section{Conclusion}
We have shown that when there is shared low-rank structure among systems, we can leverage data from other systems to help estimate individual parameters, even in the regime $r \leq  T < d$, in which systems would otherwise be non-identifiable from their own data alone.
We have presented a method to estimate the common low dimensional subspace as well as the system parameters, by a series of three least squares optimization problems, one of which can be solved simply by singular value decomposition. 
We then provided finite sample estimation error guarantees of a truncating variant our proposed method. These sample complexity results are not necessarily optimal, and we seek to better understand the trade-offs in the number of systems $N$, and the number of observations per system $T$ in the best achievable estimation error. However, experiments suggest that the three-step estimation procedure may be applied successfully to more general settings such as time-series estimation.

\bibliographystyle{IEEEtran}
\bibliography{regression_problem_bib}

\begin{thebibliography}{10}
\providecommand{\url}[1]{#1}
\csname url@samestyle\endcsname
\providecommand{\newblock}{\relax}
\providecommand{\bibinfo}[2]{#2}
\providecommand{\BIBentrySTDinterwordspacing}{\spaceskip=0pt\relax}
\providecommand{\BIBentryALTinterwordstretchfactor}{4}
\providecommand{\BIBentryALTinterwordspacing}{\spaceskip=\fontdimen2\font plus
\BIBentryALTinterwordstretchfactor\fontdimen3\font minus
  \fontdimen4\font\relax}
\providecommand{\BIBforeignlanguage}[2]{{%
\expandafter\ifx\csname l@#1\endcsname\relax
\typeout{** WARNING: IEEEtran.bst: No hyphenation pattern has been}%
\typeout{** loaded for the language `#1'. Using the pattern for}%
\typeout{** the default language instead.}%
\else
\language=\csname l@#1\endcsname
\fi
#2}}
\providecommand{\BIBdecl}{\relax}
\BIBdecl

\bibitem{dewdney2017electronic}
S.~B. Dewdney and J.~Lachance, ``Electronic records, registries, and the
  development of “big data”: crowd-sourcing quality toward knowledge,''
  \emph{Frontiers in oncology}, vol.~6, p. 268, 2017.

\bibitem{faria2010fitting}
S.~Faria and G.~Soromenho, ``Fitting mixtures of linear regressions,''
  \emph{Journal of Statistical Computation and Simulation}, vol.~80, no.~2, pp.
  201--225, 2010.

\bibitem{li2018learning}
Y.~Li and Y.~Liang, ``Learning mixtures of linear regressions with nearly
  optimal complexity,'' in \emph{Conference On Learning Theory}.\hskip 1em plus
  0.5em minus 0.4em\relax PMLR, 2018, pp. 1125--1144.

\bibitem{kong2020robust}
W.~Kong, R.~Somani, S.~Kakade, and S.~Oh, ``Robust meta-learning for mixed
  linear regression with small batches,'' \emph{Advances in neural information
  processing systems}, vol.~33, pp. 4683--4696, 2020.

\bibitem{kong2020meta}
W.~Kong, R.~Somani, Z.~Song, S.~Kakade, and S.~Oh, ``Meta-learning for mixed
  linear regression,'' in \emph{International Conference on Machine
  Learning}.\hskip 1em plus 0.5em minus 0.4em\relax PMLR, 2020, pp. 5394--5404.

\bibitem{gr1}
R.~Gribonval, R.~Jenatton, and F.~Bach, ``Sparse and spurious: dictionary
  learning with noise and outliers,'' \emph{IEEE Transactions on Information
  Theory}, vol.~61, no.~11, pp. 6298--6319, 2015.

\bibitem{gr2}
R.~Gribonval, R.~Jenatton, F.~Bach, M.~Kleinsteuber, and M.~Seibert, ``Sample
  complexity of dictionary learning and other matrix factorizations,''
  \emph{IEEE Transactions on Information Theory}, vol.~61, no.~6, pp.
  3469--3486, 2015.

\bibitem{mairal2009online}
J.~Mairal, F.~Bach, J.~Ponce, and G.~Sapiro, ``Online dictionary learning for
  sparse coding,'' in \emph{Proceedings of the 26th annual international
  conference on machine learning}, 2009, pp. 689--696.

\bibitem{tripu2021provable}
N.~Tripuraneni, C.~Jin, and M.~Jordan, ``Provable meta-learning of linear
  representations,'' in \emph{International Conference on Machine
  Learning}.\hskip 1em plus 0.5em minus 0.4em\relax PMLR, 2021, pp.
  10\,434--10\,443.

\bibitem{duchi2022subspace}
J.~C. Duchi, V.~Feldman, L.~Hu, and K.~Talwar, ``Subspace recovery from
  heterogeneous data with non-isotropic noise,'' \emph{Advances in Neural
  Information Processing Systems}, vol.~35, pp. 5854--5866, 2022.

\bibitem{pontil2015benefit}
M.~Pontil, ``The benefit of multitask representation learning,'' \emph{Machine
  Learning with Interdependent and Non-identically Distributed Data}, p.~46,
  2015.

\bibitem{du2020few}
S.~S. Du, W.~Hu, S.~M. Kakade, J.~D. Lee, and Q.~Lei, ``Few-shot learning via
  learning the representation, provably,'' \emph{arXiv preprint
  arXiv:2002.09434}, 2020.

\bibitem{Chikuse2003}
Y.~Chikuse, \emph{Statistics on Special Manifolds}.\hskip 1em plus 0.5em minus
  0.4em\relax Springer, 2003.

\bibitem{golub2013matrix}
G.~H. Golub and C.~F. Van~Loan, \emph{Matrix computations}.\hskip 1em plus
  0.5em minus 0.4em\relax JHU press, 2013.

\bibitem{ye2016schubert}
K.~Ye and L.-H. Lim, ``Schubert varieties and distances between subspaces of
  different dimensions,'' \emph{SIAM Journal on Matrix Analysis and
  Applications}, vol.~37, no.~3, pp. 1176--1197, 2016.

\bibitem{bhatia}
R.~Bhatia, \emph{Matrix Analysis}, ser. Graduate Texts in Mathematics.\hskip
  1em plus 0.5em minus 0.4em\relax New York, NY: Springer, 1997.

\bibitem{arora2014new}
S.~Arora, R.~Ge, and A.~Moitra, ``New algorithms for learning incoherent and
  overcomplete dictionaries,'' in \emph{Conference on Learning Theory}.\hskip
  1em plus 0.5em minus 0.4em\relax PMLR, 2014, pp. 779--806.

\bibitem{gribonval2010dictionary}
R.~Gribonval and K.~Schnass, ``Dictionary identification—sparse
  matrix-factorization via $\ell_1$-minimization,'' \emph{IEEE Transactions on
  Information Theory}, vol.~56, no.~7, pp. 3523--3539, 2010.

\bibitem{Bryc1995}
W.~Bryc, \emph{The Normal Distribution: Characterizations with
  Applications}.\hskip 1em plus 0.5em minus 0.4em\relax Springer, 1995.

\bibitem{meckes2019random}
E.~S. Meckes, \emph{The random matrix theory of the classical compact
  groups}.\hskip 1em plus 0.5em minus 0.4em\relax Cambridge University Press,
  2019, vol. 218.

\bibitem{Collins2014}
B.~Collins, S.~Matsumoto, and N.~Saad, ``Integration of invariant matrices and
  moments of inverses of {G}inibre and {W}ishart matrices,'' \emph{Journal of
  Multivariate Analysis}, vol. 126, pp. 1--13, 2014.

\bibitem{Roman12}
R.~Vershynin, ``How close is the sample covariance matrix to the actual
  covariance matrix?'' \emph{Journal of Theoretical Probability}, vol.~25, pp.
  655--686, Sep. 2012.

\bibitem{rudelson2009smallest}
M.~Rudelson and R.~Vershynin, ``Smallest singular value of a random rectangular
  matrix,'' \emph{Communications on Pure and Applied Mathematics: A Journal
  Issued by the Courant Institute of Mathematical Sciences}, vol.~62, no.~12,
  pp. 1707--1739, 2009.

\end{thebibliography}

\newpage
\appendix 

\section{Uniform distribution on the Grassmanian}\label{sec:uniform}

In this section, we collect some useful statements about the uniform
distribution on the Grassmanian. We identify $\Gr(r,d)$ with the set of
orthogonal projections of rank $T$:
\begin{align}
	\Gr(r,d) \cong \set{P\in\R^{d\by d}\given P=P^\T=P^2\land \Tr P = r}.
\end{align}

Under this identification, a matrix $P\in\R^{d\by d}$ is uniformly distributed
over $\Gr(r,d)$ iff $\P[P\in\Gr(r,d)] = 1$ and $Q^\T PQ\equalsd P$ for each
$Q\in\Oscr(d)$ \cite[Chap.~2]{Chikuse2003}. In particular, when $V$ is uniformly distributed over the
Stiefel manifold $\St(r,d)$, the projection $P_V = VV^\T$ is uniformly
distributed over $\Gr(r,d)$. The following lemma shows that we also obtain a
uniformly distributed element of $\Gr(r,d)$ by considering the projection onto
the linear span of rotationally invariant and independent vectors.

\begin{lemma}\label{lemma:uniform-g}
	Let $X\in\R^{T\by d}$ be a random matrix with $T\leq d$ and such that
	\begin{enumerate}
		\item each row of $X$ has a rotationally invariant distribution over
			$\R^d$ with absolutely continuous marginals\footnote{A necessary
				and sufficient condition for a rotationally invariant random
				vector $u\in\R^d$ to have absolutely continuous marginals is
				that $\P(u=0)=0$ \cite[Lemma 4.1.6]{Bryc1995}.},
		\item the rows of $X$ are mutually independent.
	\end{enumerate}
	Then the orthogonal projection $P_X$ onto $\rowspan(X)$ is uniformly distributed over $\Gr(T, d)$.
\end{lemma}

\begin{proof}
We first prove that $P_X$ is supported on $\Gr(T,d)$, or equivalently that $X$
has full row rank almost surely. Denote by $X_i$ the $i$th row of $X$ and by
$X_{-i}$ the $T-1$ remaining rows. It is sufficient to establish that
$\P[X_i\in\spn(X_{-i})]=0$ for each $i\in[N]$. Since $\spn(X_{-i})$ is a
subspace of dimension at most $d-1$, it is contained in a hyperplane of $\R^d$.
Let $u$ be a normal vector to this hyperplane with $\norm{u}_2=1$, then we have
\begin{align}
	\P[X_i\in\spn(X_{-i})]
	&\leq \P[\brk{u, X_i}=0]\\
	&= \Exp[\big]{\P{\ip{u}{X_i}=0\given X_{-i}}}\\
	&= \P[X_{i1}=0]=0.
\end{align}
The first equality is the tower rule for conditional expectations. The second
equality uses independence of the rows and the fact that by rotational
invariance the distribution of $\brk{u, X_i}$ does not depend on the unit norm
vector $u$ (and in particular is the same as $\brk{e_1, X_i}=X_{i1}$). The last
equality follows from the absolute continuity of $X_{i1}$.

Using that $P_X=\pseudo{X}X$, we have for $Q\in\Oscr(d)$
\begin{align}
	Q^\T P_X Q = Q^\T \pseudo{X}XQ = \pseudo{(XQ)} XQ \equalsd \pseudo{X}X=P_X,
\end{align}
where the penultimate equality uses that $XQ\equalsd X$ by rotational
invariance of the rows of $X$. This concludes the proof since the identity
$Q^\T P_XQ\equalsd P_X$ uniquely characterizes the uniform distribution among
distributions supported on $\Gr(T, d)$.
\end{proof}

\begin{lemma}\label{lem:1}
	Let $P$ be uniformly distributed over $\Gr(T,d)$, then $\E[P] = \frac T d
	I_d$. In particular, $\E\big[\norm{P\beta}_2^2\big]=\frac T
	d\norm{\beta}_2^2$ for each $\beta\in\R^d$, and
	$\E\big[\Fnorm{PB}^2\big]=\frac T d\Fnorm{B}^2$ for each $B\in\R^{d\by
	r}$.
\end{lemma}

\begin{proof}
	The uniform distribution over $\Gr(T,d)$ is invariant under the conjugacy
	action of $\Oscr(d)$, so the same is true for the expectation $\E[P]$. It
	is a standard fact that the only matrices that are invariant under the
	conjugacy action of $\Oscr(d)$—or equivalently, that commute with all
	matrices in $\Oscr(d)$—are the scalar matrices. Hence $\E[P] =
	\lambda I_d$ for some $\lambda\in\R$. We determine the value of $\lambda$
	by taking the trace
	\begin{align}
		\lambda d = \Tr(\lambda I_d) = \Tr\E[P] = \E[\Tr P] = T.
	\end{align}

	For the second claim
	\begin{align}
		\E\big[\norm{P\beta}_2^2\big]
		&= \E\big[\ip{P\beta}{P\beta}\big]
		= \E\big[\ip{P\beta}{\beta}\big]\\
		&= \ip{\E[P]\beta}{\beta} = \frac T d \norm{\beta}_2^2.
	\end{align}
	where the second equality uses idempotence of $P$.
	The final claim follows from the previous one by summing over the columns
	of $B$.
\end{proof}

\begin{lemma}\label{lemma:uniform-moment}
	Let $P$ be uniformly distributed over $\Gr(T,d)$. Then for all
	$P_0\in\Gr(r,d)$ 
	\begin{align}
		\E[PP_0P]= \begin{cases}\frac{T^2d+T(d-2)}{(d-1)d(d+2)}P_0
			+\frac{rT(d-T)}{(d-1)d(d+2)}I_d&\text{if $d\geq 2$}\\
			TP_0&\text{if $d=1$}
	\end{cases}.
	\end{align}
\end{lemma}

\begin{proof}
	Denote by $E_{ij}\eqdef  e_i e_j^\T$ the matrix in $\R^{d\by
	d}$ whose only non-zero entry, at $(i,j)$, equals 1. Let
	$J_r \eqdef\sum_{i=1}^r E_{ii}$ be the orthogonal projection onto the first
	$r$ canonical basis vectors of $\R^d$. Writing $P_0$ = $Q^\T J_r Q$ for
	some $Q\in\Oscr(d)$, it follows from $P\equalsd Q^\T PQ$ that
	\begin{align}
		\E[PP_0P] = Q^\T\E[PJ_rP]Q.
	\end{align}
	Hence, it is sufficient to prove the result for the matrix $J_r$. By
	linearity of $P_0\mapsto \E[PP_0P]$, we focus on computing $\E[PE_{ii}P]$
	for some $i\in[r]$. We have for indices $j,k\in[d]$
	\begin{align}
		(PE_{ii}P)_{jk} = (Pe_i)_j(Pe_i)_k= P_{ji}P_{ki}.
	\end{align}
	Furthermore, we can write $P=VV^\T$ where $V$ is uniformly distributed over
	$\St(T,d)$. Hence
	\begin{align}
		\E[PE_{ii}P]_{jk}
		= \sum_{1\leq l,l'\leq T} \E[V_{il}V_{jl}V_{il'}V_{kl'}].
	\end{align}

	We use \cite[Lemma 2.22]{meckes2019random} to compute the summand
	expectations\footnote{More generally, closed-form expressions are known for
	arbitrary monomials in entries of a uniformly random orthogonal matrix.
These can be expressed in terms of the so-called Weingarten functions
\cite[Proposition 2.2]{Collins2014}.}. If $k\neq j$, the
expectation is always zero, so we focus on the case $k=j$. If $i=j$,
	\begin{align}
		\E[PE_{ii}P]_{ii}
		&= \sum_{l=1}^T \E[V_{il}^4]
		+\sum_{l\neq l'} \E[V_{il}^2V_{il'}^2]\\
		&= \frac{3T}{d(d+2)} + \frac{T(T-1)}{d(d+2)}
		= \frac{T(T+2)}{d(d+2)}.
	\end{align}
	For $i\neq j$ we get
	\begin{align}
		\E[PE_{ii}P]_{jj}
		&= \sum_{l=1}^T \E[V_{il}^2V_{jl}^2]
		+\sum_{l\neq l'} \E[V_{il}V_{il'}V_{jl}V_{jl'}]\\
		&= \frac{T}{d(d+2)} - \frac{T(T-1)}{(d-1)d(d+2)}\\
		&= \frac{T(d-T)}{(d-1)d(d+2)}.
	\end{align}
	In summary,
	\begin{align}
		\E[PE_{ii}P]= \frac{T^2d+T(d-2)}{(d-1)d(d+2)}E_{ii}
		+\frac{T(d-T)}{(d-1)d(d+2)}I_d.
	\end{align}
	This concludes the proof after summing the previous equality for $i\in[r]$.
\end{proof}

\section{Proof of Theorem~\ref{prop:sample-complexity}}\label{app:proof}

Recall the following expression for the first-step estimate:
\begin{equation}
	\iest_i = P_{X_i}\gtruth_i + \pseudo X_i w_i = P_{X_i}B_0\phi
	+ \pseudo X_i w_i, 
\end{equation}
and the truncated first-step estimate:
\begin{equation}
	\test_i = \ind{\opnorm{\pseudo X_i}\leq s}\iest_i
\end{equation}
for some threshold $s$ that will be set at a later stage. The truncated first-step estimate provides us an upper bound on $\opnorm{\pseudo X_i}$ which will be used throughout our analysis. 

Note that since $X_i$ has full row rank almost surely, we have $\opnorm{\pseudo X_i}=1/\sigma_T(X_i)$ almost surely, where $\sigma_T(X_i)$ is the $T$th largest singular value of $X_i$. This follows immediately from the fact that the non-zero singular values of $\pseudo X_i$ are the inverse of the non-zero singular values of $X_i$. Hence we define and express the threshold probability
\begin{equation}
	p_s \eqdef \Pr{\opnorm{\pseudo X_i}\leq s} = \Pr{\sigma_T(X_i)\geq 1/s}.
\end{equation}

The next lemma shows that the truncated first step estimates are sub-Gaussian.
We will use the following standard facts about sub-Gaussian vectors.
\begin{itemize}
	\item for two independent random vectors
$X_1\in\sg_d(\sigma_1^2)$ and $X_2\in\sg_d(\sigma_2^2)$, we have
$X_1+X_2\in\sg_d(\sigma_1^2+\sigma_2^2)$.
\item for $X\in\sg_d(\sigma^2)$ and $A\in\R^{n\times d}$, we have $AX\in\sg_n(\opnorm A^2\sigma^2)$.
\item if $X\in\sg_d(\sigma^2)$, then for all $t>0$ and $\alpha\in\R^d$
	\begin{equation}
		\Pr{\ip X\alpha>t}\leq \exp\left(-\frac{t^2\enorm\alpha^2}{2\sigma^2}\right)
	\end{equation}
\end{itemize}

\begin{lemma}\label{lemma:subg}
	For all unit vectors $u\in S^{d-1}$ and $t>0$,
	\begin{equation}
	\Pr*{\abs[\big]{\ip{\test_i}{u}}>t}\leq 2e^{-t^2/L^2},
	\end{equation}
	with $L^2 = 2(\sigma_\phi^2+s^2\sigma_w^2)$.
\end{lemma}

\begin{proof}
	Consider a unit vector $u\in S^{d-1}$ and $t>0$. By the law of total expectation
	\begin{equation}\label{eq:subg-te}
	\Pr[\big]{\ip{\test_i}{u}>t}
		= \Exp*{\ind{\opnorm{\pseudo X_i}\leq s}\Pr[\big]{\ip{\iest_i} u> t\given
		X_i}}.
	\end{equation}

	Conditioned on $X_i$, $\iest_i = P_{X_i}B_0\phi + \pseudo X_i w_i$ is the
	sum of two independent sub-Gaussian variables, with
	$P_{X_i}B_0\phi\in\sg(\opnorm{P_{X_i}B_0}^2\sigma_\phi^2)$ and  $\pseudo
	X_iw_i\in\sg(\opnorm{\pseudo X_i}^2\sigma_w^2)$.

	Since $\opnorm{P_{X_i} B_0}\leq \opnorm{P_{X_i}}\opnorm {B_0}\leq
	1$, we have $\iest_i\mid X_i\in\sg(\sigma_\phi^2+\opnorm{\pseudo
	X_i}\sigma_w^2)$. Hence
	\begin{equation}
		\Pr*{\ip{\iest_i} u> t\given X_i}\leq
		\exp\left(-\frac{t^2}{2\big(\sigma_\phi^2+\opnorm{\pseudo
		X_i}^2\sigma_w^2\big)}\right).
	\end{equation}
	Whenever $\opnorm{\pseudo X_i}\leq s$, the term on the right-hand side in the previous
	inequality is upper-bounded by $e^{-t^2/L^2}$, which implies by
\eqref{eq:subg-te} that $\Pr[\big]{\ip{\test_i}{u}>t}\leq e^{-t^2/L^2}$.
Replacing $u$ with $-u$ we obtain  $\Pr[\big]{\ip{\test_i}{u}<-t}\leq
	e^{-t^2/L^2}$ and we conclude by taking a union bound.
\end{proof}

\begin{lemma}\label{lemma:covariance}
Let $\hat \Sigma_\beta\eqdef\frac 1 N \sum_{i=1}^N \test_i{\test_i}^\T$ denote the empirical covariance matrix of the truncated first step estimates, and define its expectation $\Sigma_\beta\eqdef \E[\hat\Sigma_\beta]$. Defining $p_s = \P[\opnorm{\pseudo X_i}\leq s]$, we have $\Sigma_\beta =
\lambda P_{\Bscr_0} + \mu I_d$ where if $T\leq d$,   
	\begin{gather}
	\lambda = \sigma_\phi^2 p_s\frac{T^2d+T(d-2)}{(d-1)d(d+2)}\\
	\mu = \frac{\sigma_\phi^2 p_s rT(d-T)}{(d-1)d(d+2)}
	+ \frac{\sigma_w^2 p_s}{d} \Exp{\Fnorm{\pseudo X_i}^2\given \opnorm{\pseudo
	X_i}\leq s},
\end{gather}
and if $T>d$, we have 
	\begin{gather}
	\lambda = \sigma_\phi^2 p_s\\
	\mu = \frac{\sigma_w^2 p_s}{d} \Exp{\Fnorm{\pseudo X_i}^2\given \opnorm{\pseudo
	X_i}\leq s},
\end{gather}
In particular, $\Bscr_0$ and $\Bscr_0^\perp$ are the two eigenspaces of $\Sigma_\beta$ with spectral gap $\lambda$, and $\Bscr_0$ is the $r$th principal subspace of $\Sigma_\beta$.
\end{lemma}

\begin{proof}
	For notational convenience, we introduce the binary indicator variable
	$Z_i\eqdef \ind{\sigma_T(X_i)\geq 1/s}$, governing the truncation of the
	first step estimate. Note that $\E[\hat\Sigma_\beta]=\E[\test_i{\test_i}^\T]$ and since $w_i$ is independent of $X_i$ and $\gtruth_i$, we get
\begin{align}\label{eq:covariance-decomposition}
	\E\big[\test_i{\test_i}^\T\big]
&= \E[Z_i P_{X_i}\gtruth_i{\gtruth_i}^\T P_{X_i}]\\
&\quad+ \E\big[Z_i\pseudo X_i w_i w_i^\T {\pseudo X_i}^\T\big]
\end{align}

We compute the first expectation in \eqref{eq:covariance-decomposition} as
\begin{align}
\E[Z_i P_{X_i}\gtruth_i{\gtruth_i}^\T P_{X_i}]
&=\E[Z_i P_{X_i}B_0\phi_i\phi_i^\T B_0^\T P_{X_i}]\\
&= \sigma_\phi^2\E[Z_iP_{X_i}P_{\Bscr_0}P_{X_i}]. \label{eq:intermed1}
\end{align} 
The first equality uses the definition of $\gtruth_i$. We integrate $\phi_i$
out in the second equality using the law of total expectation and isotropy of
$\phi_i$.

If $T\leq d$, \eqref{eq:intermed1} then becomes 
\begin{align}
\sigma_\phi^2\E[Z_iP_{X_i}P_{\Bscr_0}P_{X_i}]
&=\sigma_\phi^2 p_s\Exp{P_{X_i}P_{\Bscr_0}P_{X_i}\given Z_i=1}\\
	&=\sigma_\phi^2 p_s\frac{T^2d+T(d-2)}{(d-1)d(d+2)} P_{\Bscr_0}\\
	&\quad + \frac{\sigma_\phi^2p_s rT(d-T)}{(d-1)d(d+2)} I_d.
\end{align}
 where we apply the law of total expectation again in the first step. Finally, it is easy to check that the distribution of $X_i$
conditioned on $Z_i=1$ is invariant under right multiplication by an element of
$\Oscr(d)$, and has full row rank with probability 1. This implies by
\cref{lemma:uniform-g}, that conditioned on $Z_i=1$, $P_{X_i}$ is uniformly
distributed over $\Gr(T, d)$, which allows us to apply \cref{lemma:uniform-moment}.

Meanwhile, if $T>d$, the rows of $X_i$ span $\R^d$ almost surely, so that $P_{X_i} = I_d$ with probability 1. Then \eqref{eq:intermed1} becomes 
\begin{align}
    \sigma_\phi^2 \E[Z_i P_{\Bscr_0}] = \sigma_\phi^2 p_s P_{\Bscr_0}.
\end{align}

For the second expectation in \eqref{eq:covariance-decomposition}, for any $T\geq 1$,
\begin{align}
	\E\big[Z_i\pseudo X_i w_i w_i^\T {\pseudo X_i}^\T\big]
	&= \sigma_w^2\E\big[Z_i\pseudo X_i{\pseudo X_i}^\T\big]\\
	&= \sigma_w^2 p_s\Exp[\big]{\pseudo X_i{\pseudo X_i}^\T\given Z_i=1}\\
	&= \frac{\sigma_w^2p_s}{d} \Exp[\big]{\Fnorm{\pseudo X_i}^2\given Z_i=1} I_d,
\end{align}
where we used isotropy of $w_i$ in the first equality and the law of total
expectation for the second equality. Finally, the identity
$\pseudo{(X_iQ^\T)}=Q\pseudo X_i$, valid for all $Q\in\Oscr(d)$, shows that, conditioned on $Z_i=1$, the columns of $\pseudo{X_i}$ have rotationally invariant distributions, due to the rows of $X_i$ having rotationally invariant distributions conditioned on $Z_i=1$.
\end{proof}

\begin{proof}[Proof of \cref{prop:sample-complexity}]
	We first apply a variant of the Davis–Kahan theorem (see e.g.\
	\cite[Thm~2.1]{duchi2022subspace}),
    to bound the maximum principal angle between subspaces $\hat\Bscr$ and $\Bscr_0$ as
    \begin{equation}\label{eq:sint}
		\opnorm{\sin \Theta(\hat\Bscr,\Bscr_0)}
		\leq 
		\frac{\opnorm{\hat\Sigma_\beta-\Sigma_\beta}}{\lambda},
	\end{equation}
	where we used that $\hat\Bscr$ is the $r$th principal subspace of $\hat \Sigma_\beta\eqdef\frac 1 N \sum_{i=1}^N \test_i{\test_i}^\T$ by definition of our estimator, and by \cref{lemma:covariance}, that $\Bscr_0$ is the
	$r$th principal subspace of the matrix $\Sigma_\beta\eqdef
	\E[\hat\Sigma_\beta]$, with spectral gap $\lambda$.
	Using the expression for $\lambda$ from
	\cref{lemma:covariance}, we see that $\lambda \geq \frac 1 3\sigma_\phi^2
	p_s T^2/d^2$ when $T\leq d$ and since $\lambda=\sigma_\phi^2 p_s$ when
	$T>d$, we have
	\begin{equation}\label{eq:spectral-gap}
		\frac 1 \lambda \lesssim \frac 1 {\sigma_\phi^2
		p_s}\max\lrp*{\frac{d^2}{T^2},1}.
	\end{equation}

	We then apply \cite[Prop.~2.1]{Roman12} which provides a tail bound for the empirical covariance matrix of independent sub-Gaussian random vectors. By \cref{lemma:subg}, the truncated estimate $\test_i$ satisfies a sub-Gaussian tail bound with variance proxy $L=\sigma_\phi^2+s^2\sigma_w^2$, hence
	\begin{equation}\label{eq:cov-concentration}
		\opnorm{\hat\Sigma_\beta-\Sigma_\beta}
		\leq L \sqrt{\frac d N\log\frac 2\delta}
		= (\sigma_\phi^2+s^2\sigma_\omega^2)
		\sqrt{\frac d N\log\frac 2\delta},
	\end{equation}
	for all $0<\delta<1$ and with probability at least $1-\delta$.

Plugging \eqref{eq:spectral-gap} and \eqref{eq:cov-concentration} into
	\eqref{eq:sint}, we obtain that for every $0<\delta<1$
	\begin{equation}
		\opnorm{\sin \Theta(\hat\Bscr,\Bscr_0)}
		\lesssim \frac{1+s^2\sigma_w^2/\sigma_\phi^2}{p_s}
	\max \lrp*{ \frac{d^2}{T^2}, 1}
		\sqrt{\frac d
		N\log\frac 2\delta},
	\end{equation}
	with probability at least $1-\delta$.

Finally, we set $s$ based on the following lower bound on $p_s$ provided by \cite[Thm~1.1]{rudelson2009smallest}. When $T\leq d$,
\begin{equation}\label{eq:rude}
	p_s\geq 1- \left(\frac{C}{s(\sqrt{d}-\sqrt{T-1})}\right)^{d-T+1}-e^{-cd}
\end{equation}
for some universal positive constants $C$ and $c$. For $1/s = (1-e^{-c})(\sqrt d - \sqrt {T-1})/(2C)$, we get
\begin{align}
	p_s &\geq 1 - \left(\frac{1-e^{-c}}{2}\right)^{d-T+1}- e^{-cd}\\
		&\geq 1 - \left(\frac{1-e^{-c}}{2}\right)- e^{-c} = \frac{1-e^{-c}} 2.
\end{align}
For this setting of $s$ when $T\leq d$, we obtain with probability $1-\delta$,
	\begin{equation}
		\opnorm{\sin \Theta(\hat\Bscr,\Bscr_0)}
		\lesssim \left(1+\frac{\sigma_w^2/\sigma_\phi^2}{(\sqrt d - \sqrt
		{T-1})^2}\right)
	\frac{d^2}{T^2}
		\sqrt{\frac d N\log\frac 2\delta}.
	\end{equation}

 When $T>d$, the bound \eqref{eq:rude} holds with the roles of $T$ and $d$ swapped. Repeating the setting of $s$ above with $T$ and $d$ swapped, we have that when $T>d$, with probability at least $1-\delta$,
	\begin{equation}
		\opnorm{\sin \Theta(\hat\Bscr,\Bscr_0)}
		\lesssim \left(1+\frac{\sigma_w^2/\sigma_\phi^2}{(\sqrt T - \sqrt {d-1})^2}\right)
		\sqrt{\frac d N\log\frac 2\delta}.
	\end{equation} 
\end{proof}

\end{document}